\documentclass[12pt,pdftex,noinfoline]{imsart}

\RequirePackage[OT1]{fontenc}
\usepackage{amsthm,amsmath,amsfonts,natbib,mathtools,amssymb}
\RequirePackage{hypernat}
\usepackage[ruled,section]{algorithm}
\usepackage{algorithmic}
\usepackage{graphicx}
\usepackage{verbatim}
\usepackage{times}
\usepackage[T1]{fontenc}
\usepackage[text={6.0in,8.6in},centering]{geometry}
\usepackage{subfigure}
\usepackage{graphicx}
\usepackage{hyperref}
\usepackage{yhmath}
\usepackage[usenames,dvipsnames]{xcolor}
\definecolor{shadecolor}{gray}{0.9}
\definecolor{shadecolor}{gray}{0.9}
\hypersetup{citecolor=MidnightBlue}
\hypersetup{linkcolor=MidnightBlue}
\hypersetup{urlcolor=MidnightBlue}
\usepackage{enumerate}
\usepackage{fullpage}

\usepackage{tikz}
\usetikzlibrary{arrows,positioning}
\usetikzlibrary{decorations.pathreplacing}

\let\hat\widehat
\let\tilde\widetilde

\newcommand{\E}{{\mathbb E}}

\newcommand{\A}{{\mathcal{A}}}

\newcommand{\C}{{\mathcal{C}}}

\newcommand{\F}{{\cal F}}

\def\argmin{\mathop{\rm argmin}}

\theoremstyle{plain}
\newtheorem{theorem}{Theorem}[section]
\newtheorem{proposition}[theorem]{Proposition}
\newtheorem{lemma}[theorem]{Lemma}
\newtheorem{corollary}[theorem]{Corollary}
\theoremstyle{remark}
\newtheorem{remark}[theorem]{Remark}
\theoremstyle{definition}
\newtheorem{example}[theorem]{Example}

\DeclareMathOperator*{\argmax}{\arg\max}
\newcommand{\T}{\text{T}}

\def\E{{\mathbb E}}
\def\supp{\mathop{\text{supp}\kern.2ex}}
\def\argmin{\mathop{\text{arg\,min}\kern.2ex}}
\let\hat\widehat
\let\tilde\widetilde

\def\err{\mbox{\rm err}}
\def\A{{\mathcal A}}
\def\C{{\mathcal C}}
\def\F{{\mathcal F}}
\def\sfrac#1#2{{#1}/{#2}}

\newcommand{\dist}{\mathop{\rm {dist}}}
\newcommand{\Proj}{\text{Proj}}

\providecommand{\dom}{\mathop{\rm dom}}
\newcommand{\simiid}{\stackrel{\rm iid}{\sim}}
\newcommand{\R}{\mathbb{R}}
\newcommand{\defeq}{:=}
\newcommand{\mc}[1]{\mathcal{#1}}
\newcommand{\half}{\frac{1}{2}}
\newcommand{\dkl}[2]{D_{\rm kl}\left({#1} |\!| {#2}\right)}
\newcommand{\what}[1]{\widehat{#1}} 
\newcommand{\norm}[1]{\left\|{#1}\right\|} 
\newcommand{\tvnorm}[1]{\norm{#1}_{\rm TV}}

\newcommand{\normal}{\mathsf{N}}
\newcommand{\modcont}{\omega}

\newcommand{\opt}{^\star}
\newcommand{\error}{\mathop{\rm error}}

\renewcommand{\>}{\right\rangle}

\providecommand{\dchii}[2]{D_{\chi}\left({#1} |\!| {#2}\right)}
\newcommand{\hinge}[1]{\left[{#1}\right]_+} 

\newcommand{\fnmetric}{\kappa}
\newcommand{\fndist}{d}
\newcommand{\stddev}{\sigma}
\newcommand{\noise}{\varepsilon}

\newcommand{\pow}{k}
\newcommand{\statrv}{\xi}

\numberwithin{equation}{section}

\begin{document}

\begin{frontmatter}
\title{Local Minimax Complexity of \\ Stochastic Convex Optimization}
\runtitle{Local Minimax Complexity of Stochastic Convex Optimization}
\runauthor{Chatterjee, Duchi, Lafferty, and Zhu}

\begin{aug}
\vskip10pt
\author{\fnms{Yuancheng}
  \snm{Zhu${}^{*}$}\ead[label=e1]{lafferty@galton.uchicago.edu}}
\;
\author{\fnms{Sabyasachi} \snm{Chatterjee${}^{*}$}\ead[label=e2]{sabyasachi@galton.uchicago.edu}}
\;
\author{\fnms{John}
  \snm{Duchi${}^{\S\ddag}$}\ead[label=e3]{lafferty@galton.uchicago.edu}}
\, \and \,
\author{\fnms{John}
  \snm{Lafferty${}^{*\dag}$}\ead[label=e4]{lafferty@galton.uchicago.edu}}
\vskip10pt
\address{
\begin{tabular}{cc}
${}^*$Department of Statistics & ${}^\S$Department of Statistics\\
${}^\dag$Department of Computer Science & ${}^\ddag$Department of
  Electrical Engineering\\
University of Chicago & Stanford University
\end{tabular}
\\[10pt]
\today\\[5pt]
\vskip10pt
}
\end{aug}

\begin{abstract}
We extend the traditional worst-case, minimax analysis of stochastic
convex optimization by introducing a localized form of minimax
complexity for individual functions.  Our main result gives
function-specific lower and upper bounds on the number of stochastic
subgradient evaluations needed to optimize either the function or its
``hardest local alternative'' to a given numerical precision.  The
bounds are expressed in terms of a localized and computational
analogue of the modulus of continuity that is central to statistical
minimax analysis. We show how the computational modulus of continuity
can be explicitly calculated in concrete cases, and relates to the
curvature of the function at the optimum.  We also prove a
superefficiency result that demonstrates it is a meaningful benchmark,
acting as a computational analogue of the Fisher information in
statistical estimation. The nature and practical implications of the
results are demonstrated in simulations.
\end{abstract}
 
\vskip20pt 
\end{frontmatter}

\maketitle

\vskip10pt
\section{Introduction}


The traditional analysis of algorithms is based on a worst-case,
minimax formulation.  One studies the running time,
measured in terms of the smallest number of arithmetic operations
required by any algorithm to solve any instance in the family of
problems under consideration.  Classical worst-case complexity
theory focuses on discrete problems.  In the setting of
convex optimization, where the problem instances require numerical
rather than combinatorial optimization, \citet{NemirovskiYudin:83}
developed an approach to minimax analysis based on a first order
oracle model of computation.  In this model, an algorithm
to minimize a convex function can make queries
to a first-order ``oracle,'' and the complexity is defined as
the smallest error achievable using some specified 
minimum number of queries needed.  Specifically, the oracle is queried with an input
point $x\in \C$ from a convex domain $\C$, and returns an unbiased estimate
of a subgradient vector to the function $f$ at $x$.
After $T$ calls to the oracle, an algorithm $A$
returns a value $\hat x_A\in\C$, which is a random variable due to
the stochastic nature of the oracle, 
and possibly also due to randomness in the algorithm.  
The Nemirovski-Yudin analysis reveals that, in
the worst case, the number of calls to the oracle required to drive
the expected error $\E(f(\hat x_A) - \inf_{x\in\C}
f(x))$ below $\epsilon$ scales as $T = O(1/\epsilon)$ for the class of
strongly convex functions, and as $T = O(1/\epsilon^2)$ for the class
of Lipschitz convex functions.

In practice, one naturally finds that some functions are easier to
optimize than others.   Intuitively, if the function is ``steep'' near the optimum, then
the subgradient may carry a great deal of information, and
a stochastic gradient descent algorithm may converge relatively
quickly.  A minimax approach to analyzing
the running time cannot take this into account for a particular function, as it
treats the worst-case behavior of the algorithm over all functions.
It would be of considerable interest to 
be able to assess the complexity of solving 
an individual convex optimization problem.  Doing
so requires a break from traditional worst-case
thinking.    

In this paper we revisit the traditional view of the 
complexity of convex optimization from the point of view of
a type of localized minimax complexity.  In local minimax,
our objective is to quantify the intrinsic difficulty of optimizing a
specific convex function $f$. With the target $f$ fixed, we
take an alternative function $g$ within the same
function class $\F$, and evaluate how the maximum expected error decays with the
number of calls to the oracle, for an optimal algorithm 
designed to optimize either $f$ or $g$.  
The local minimax complexity $R_T(f;\F)$ is
defined as the least favorable alternative $g$:
\begin{equation}
R_T(f;\F) = \sup_{g\in\F} \;\inf_{A\in\A_T} \;\max_{h\in\{f,g\}}\,\error(A,h)
\end{equation}
where $\error(A,h)$ is some measure of error for the algorithm applied to function $h$.
In contrast, the traditional
global worst-case performance of the best algorithm, as defined by the 
minimax complexity $R_{T}(\F)$ of Nemirovsky and Yudin, is
\begin{equation}
R_T(\F) = \inf_{A\in\A_T} \;\sup_{g\in\F}\, \error(A,g).
\end{equation}
The local minimax complexity can be thought of as the difficulty of
optimizing the hardest alternative to the target function.
Intuitively, a difficult alternative is a function $g$ for which
querying the oracle with $g$ gives results similar to querying with
$f$, but for which the value of $x\in\C$ that minimizes $g$ is far
from the value that minimizes $f$.

Our analysis ties this function-specific notion of complexity to a
localized and computational analogue of the modulus of continuity that
is central to statistical minimax analysis
\citep{DonohoLiu87GR1,DonohoLiu90GR2}.  We show that the local minimax
complexity gives a meaningful benchmark for quantifying the difficulty
of optimizing a specific function by proving a superefficiency
result; in particular, outperforming this benchmark at some function
must lead to a larger error at some other function.  Furthermore, we
propose an adaptive algorithm in the one-dimensional case that is
based on binary search, and show that this algorithm automatically
achieves the local minimax complexity, up to a logarithmic factor.
Our study of the algorithmic complexity of convex optimization is
motivated by the work of \citet{CaiLow15}, who propose an analogous
definition in the setting of statistical estimation of a
one-dimensional convex function.  The present work can thus be seen as
exposing a close connection between statistical estimation and
numerical optimization of convex functions. In particular, our results
imply that the local minimax complexity can be viewed as a computational
analogue of Fisher information in classical statistical estimation.

In the following section we establish our notation, and give a 
technical overview of our main results, which characterize the local minimax
complexity in terms of the computational modulus of continuity.
In Section~\ref{sec:superefficiency}, we demonstrate the phenomenon of 
superefficiency of the local minimax complexity.
In Section~\ref{sec:algor} we present the algorithm that adapts to the 
benchmark, together with an analysis of its theoretical properties.
We also present simulations of the algorithm and comparisons
to traditional stochastic gradient descent.
Finally, we conclude with a brief review of related work and a discussion
of future research directions suggested by our results.

\section{Local minimax complexity}\label{sec:background}

In this section, we first establish notation and define a modulus of
continuity for a convex function $f$.  We then state our main result,
which links the local minimax complexity to this modulus of
continuity.

Let $\mathcal F$ be the collection of Lipschitz convex functions 
defined on a compact convex set $\C\subset\mathbb R^d$.
Given a function $f\in\F$, our goal is to find a minimum point,
$x_f^*\in\argmin_{x\in \C} f(x)$.
However, our knowledge about $f$ can only be gained through a 
first-order oracle. The oracle, upon being queried with $x\in\C$, returns 
$f'(x)+\xi$, where $f'(x)$ is a subgradient of $f$ at $x$ and 
$\xi\sim\normal(0,\sigma^2I_d)$.
When the oracle is queried with a non-differentiable point $x$ of $f$, 
instead of allowing the oracle to return an arbitrary subgradient at $x$,
we assume that it has a deterministic mechanism for producing $f'(x)$.
That is, when we query the oracle with $x$ twice, it should return two 
random vectors with the same mean $f'(x)$.
Such an oracle can be realized, for example, by taking
$f'(x) = \argmin_{z\in\partial f(x)}\|z\|$.

Consider optimization algorithms that make a total of $T$
queries to this first-order oracle, and let $\mathcal A_T$ be the collection
of all such algorithms. For $A\in\mathcal A_T$, denote by $\hat x_A$
the output of the algorithm. 
We write $\err(x,f)$ for a measure of error for using $x$ as
the estimate of the minimum point of $f\in\F$. In this notation,
the usual minimax complexity is defined as
\begin{equation}
R_T(\F) = \inf_{A\in\mathcal A_T}\sup_{f\in\mathcal F}\;\mathbb E_f\,\err(\hat x_A,f).
\end{equation}
Note that the algorithm $A$ queries
the oracle at up to $T$ points $x_t\in\C$ selected sequentially, and the
output $\hat x_A$ is thus a 
function of the entire sequence of random vectors $v_t \sim N(f'(x_t),
\sigma^2I_d)$ returned by the oracle. The expectation $\E_f$ denotes the
average with respect to this randomness (and any additional randomness
injected by the algorithm itself).
The minimax risk $R_T(\F)$ characterizes the hardness of the entire class $\mathcal F$. 
To quantify the difficulty of optimizing an individual function $f$, we consider the
following local minimax complexity, comparing $f$ to its hardest local alternative
\begin{equation}\label{eqn:localminimax}
R_T(f;\F) = \sup_{g\in\F}\inf_{A\in{\mathcal A_T}}\max_{h\in\{f,g\}}\,\mathbb E_h\,\err(\hat x_A,h).
\end{equation}

We now proceed to define a computational modulus of continuity that
characterizes the local minimax complexity.
Let $\mc{X}_f^*=\argmin_{x\in\C} f(x)$ be the set of minimum points of function $f$.
We consider $\err(x,f)=\inf_{y\in\mc{X}_f^*}\|x-y\|$ as our measure of error.
Define $d(f,g) = \inf_{x\in\mc{X}_f^*,y\in\mc{X}_g^*}\|x-y\|$ for $f,g\in\mathcal F$.
It is easy to see that $\err(x,f)$ and $d(f,g)$ satisfy the
\emph{exclusion inequality}
\begin{equation}\label{eqn:exclusion-inequality}
\err(x,f)<\frac{1}{2}d(f,g)\text{\quad implies\quad}\err(x,g)\geq \frac{1}{2}d(f,g).
\end{equation}
Next we define
\begin{equation}
\kappa(f,g) = \sup_{x\in\C}\|f'(x)-g'(x)\|
\end{equation}
where $f'(x)$ is the unique subgradient of $f$ that is returned as the mean 
by the oracle when queried with $x$.
For example, if we take $f'(x)=\argmin_{z\in\partial f(x)}\|z\|$, we have
\begin{equation}
\kappa(f,g) = \sup_{x\in\mc{C}}\|\Proj_{\partial f(x)}(0)-\Proj_{\partial g(x)}(0)\|
\end{equation}
where $\Proj_{B}(z)$ is the projection of $z$ to the set $B$.
Thus, $d(f,g)$ measures the dissimilarity between two functions 
in terms of the distance between their minimizers, 
whereas $\kappa(f,g)$ measures the dissimilarity by 
the largest separation between their subgradients at any given point.

Given $d$ and $\kappa$, we define the \emph{modulus of continuity} of $d$
with respect to $\kappa$ at the function $f$ by 
\begin{equation}\label{eqn:modulus}
\omega_f(\epsilon) = \sup\left\{d(f,g):g\in\F,\kappa(f,g)\leq\epsilon\right\}.
\end{equation}

\begin{figure}
\begin{center}
\input{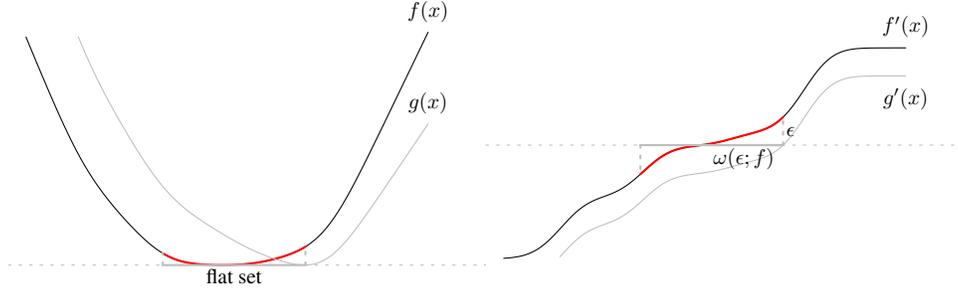}
\end{center}
\caption{Illustration of the flat set and the modulus of continuity. 
Both the function $f$ (left) and its derivative $f'$ (right) are shown (black curves), along with 
one of the many possible alternatives, $g$ and its derivative $g'$ (solid gray curves), 
that achieve the $\sup$ in the definition of $\omega_f(\epsilon)$.
The flat set contains all the points for which $|f'(x)|< \epsilon$,
and $\omega_f(\epsilon)$ is the larger half width of the flat set.}\label{fig:flatset}
\end{figure}

We now show how to calculate the modulus for some specific functions.
\begin{example}\label{exmp:one-dim-modulus}
Suppose that $f$ is a convex function on a one-dimensional interval
$\C\subset \mathbb R$. Then we have
\begin{equation}
\omega_f(\epsilon)=\sup\left\{\inf_{x\in\mc{X}_f^*}|x-y|:y\in\C,|f'(y)|<\epsilon\right\}.\label{eqn:rho}
\end{equation}
This essentially says that the modulus of continuity measures the
size (in fact, the larger half-width) of the the ``flat set'' where the magnitude of the subderivative
is smaller than $\epsilon$. See Figure~\ref{fig:flatset} for an illustration
Thus, for the class of symmetric functions 
$f(x)=\frac{1}{\pow}|x|^\pow$ over $\C=[-1,1]$, with $\pow>1$, 
\begin{equation}
\omega_f(\epsilon)=\epsilon^{\frac{1}{\pow-1}}.
\end{equation}
For the asymmetric case $f(x)=\frac{1}{\pow_l}|x|^{\pow_l}I(-1\leq x\leq 0)+\frac{1}{\pow_r}|x|^{\pow_r}I(0< x\leq 1)$
with $\pow_l, \pow_r>1$,
\begin{equation}
\omega_f(\epsilon) = \epsilon^{\frac{1}{\pow_l\lor \pow_r-1}}.
\end{equation}
That is, the size of the flat set depends on the flatter side of the function.
\end{example}

\subsection{Local minimax is characterized by the modulus}

We now state our main result linking  the local minimax complexity 
to the modulus of continuity. We say that the modulus of the continuity has \emph{polynomial growth} if 
there exists $\alpha>0$ and $\epsilon_0$, such that for any $c\geq 1$ and $\epsilon\leq\epsilon_0/c$
\begin{equation}
\omega_f(c\epsilon)\leq c^\alpha\omega_f(\epsilon).\label{eqn:polygrowth}
\end{equation}
Our main result below shows that the modulus of continuity characterizes the local minimax complexity of optimization of
a particular convex function, in a manner similar to how the modulus of continuity quantifies
the (local) minimax risk in a statistical estimation setting
\cite{CaiLow15,DonohoLiu87GR1,DonohoLiu90GR2}, relating the objective
to a geometric property of the function.

\begin{theorem}\label{thm:localminimax}
Suppose that $f\in\F$ and that $\omega_f(\epsilon)$ has polynomial growth.
Then there exist constants $C_1$ and $C_2$ independent of $T$ and $T_0>0$ such that for all $T>T_0$
\begin{equation}
C_1\,\omega_f\left(\frac{\sigma}{\sqrt{T}}\right)\leq R_T(f;\F)\leq C_2\,\omega_f\left(\frac{\sigma}{\sqrt{T}}\right).
\label{eq:thm}
\end{equation}
\end{theorem}

\begin{remark}
We use the error metric $\err(x,f)=\inf_{y\in\mc{X}^*_f}\|x-y\|$ here. 
For a given a pair $(\err,d)$ that satisfies the exclusion inequality \eqref{eqn:exclusion-inequality},
our proof technique applies to yield the corresponding lower bound.
For example, we could use $\err(x,f) = \inf_{y\in\mc{X}_f^*}|v^T(x-y)|$ for some vector $v$.
This error metric would be suitable when we wish to estimate $v^Tx_f^*$, 
for example, the first coordinate of $x_f^*$.
Another natural choice of error metric is $\err(x,f) = f(x)-\inf_{x\in\C}f(x)$,
with a corresponding distance $d(f,g) = \inf_{x\in\C}|f(x)-\inf_xf(x)+g(x)-\inf_xg(x)|$.
For this case, while the proof of the lower bound stays exactly the same,
further work is required for the upper bound, which is beyond the scope of this paper. 
\end{remark}


\begin{remark}
Although the theorem gives an upper bound for the local minimax complexity, 
this does not guarantee the existence of an algorithm that achieves
the local complexity for any function. Therefore, it is important to
design an algorithm that adapts  to this benchmark for each individual function. 
We solve this problem in the one-dimensional case in Section~\ref{sec:algor}.
\end{remark}

The proof of this theorem is given in the appendix. We now illustrate
the result with examples that verify the intuition that different functions should have
different degrees of difficulty for stochastic convex optimization.

\begin{example}
For the function $f(x)=\frac{1}{\pow}|x|^\pow$ with $x\in[-1,1]$ for $\pow>1$, we have
$R_T(f;\mathcal F) = O\bigl(T^{-\frac{1}{2(\pow-1)}}\bigr)$.
When $\pow=2$, we recover the strongly convex case, where the (global) minimax 
complexity is $O\bigl(1/\sqrt{T}\bigr)$ with respect to the error 
$\err(x,f)=\inf_{y\in\mc{X}_f^*}\|x-y\|$. We see a faster rate of convergence for $\pow<2$.  
As $\pow\to\infty$, we also see that the error fails to decrease as $T$ gets large.
This corresponds to the worst case for any Lipschitz convex function. 
In the
asymmetric setting with $f(x)=\frac{1}{\pow_l}|x|^{\pow_l}I(-1\leq x\leq 0)+\frac{1}{\pow_r}|x|^{\pow_r}I(0< x\leq 1)$
with $\pow_l,\pow_r>1$, we have $R_T(f;\mathcal F) = O(T^{-\frac{1}{2(\pow_l\lor \pow_r-1)}})$.
\end{example}

The following example illustrates that the local minimax complexity
and modulus of continuity are consistent with known behavior of
stochastic gradient descent for strongly convex functions.

\begin{example}
In this example we consider the error $\err(x,f) = \inf_{y\in\mc{X}_f^*}|v^T(x-y)|$
for some vector $v$, and let $f$ be an arbitrary convex function satisfying 
$\nabla^2 f(x^*_f) \succ 0$ with Hessian continuous around 
$x^*_f$. Thus the optimizer $x_f^*$ is unique.
If we define $g_w(x) = f(x) - w^T \nabla^2 f(x^*_f) x$, then $g_w(x)$ 
is a convex function with unique minimizer and
\begin{equation}
\fnmetric(f, g_w) 
 = \sup_x \left\{\norm{\nabla f(x) - (\nabla f(x) - \nabla^2 f(x^*_f) w )}\right\} 
 = \norm{\nabla^2 f(x^*_f) w}.
\end{equation}
Thus, defining $\delta(w) = x_f^* - x_{g_w}^*$,
\begin{equation}
  \modcont_f \left(\frac{\sigma}{\sqrt{T}} \right)
  \geq \sup_w\{|v^T \delta(w) | :
  \norm{ \nabla^2 f(x^*_f) w} \le \sfrac{\sigma}{\sqrt{T}}\}
  \geq \sup_u
  \left|v^T \delta \left( \frac{\sigma}{\sqrt{T}} 
  \nabla^2 f(x^*_f)^{-1} u \right) \right|.
\end{equation}
By the convexity of $g_w$, we know that $x_{g_w}^*$ satisfies 
$\nabla f( x_{g_w}^* ) - \nabla^2 f(x^*_f)^{-1} w = 0$,
and therefore by the implicit function theorem,
$x_{g_w}^* = x_f^* + w + o(\norm{w})$ as $w \rightarrow 0$.
Thus, 
 \begin{equation}
   \modcont_f \left(\frac{\sigma}{\sqrt{T}} \right) 
   \geq  \frac{\sigma}{\sqrt{T}}\norm{\nabla^2 f(x^*_f)^{-1}v}
   + o\left(\frac{\sigma}{\sqrt{T}}\right) ~~\mbox{as}~ T\rightarrow \infty.
 \end{equation}
In particular, we have the local minimax lower bound
\begin{equation}
\liminf_{T \to \infty} \sqrt{T}
  R_T(f;\F) \ge C_1 \sigma \norm{\nabla^2 f(x^*_f)^{-1}v}
\end{equation}
where $C_1$ is the same constant appearing in Theorem~\ref{thm:localminimax}.
This shows that the local minimax complexity 
captures the function-specific dependence on the constant in the
strongly convex case. Stochastic gradient descent with averaging is known
to adapt to this strong convexity constant \citep{ruppert:88,PolyakJu92,moulines:11}.

\end{example}

\subsection{Superefficiency}\label{sec:superefficiency}

Having characterized the local minimax complexity in terms of a computational
modulus of continuity, we would now like to show that there are consequences to
outperforming it at some function.  This will strengthen the case that the
local minimax complexity serves as a meaningful benchmark
to quantify the difficulty of optimizing any particular convex function.

Suppose that $f$ is any one-dimensional function such
that $\mc{X}_f^* = [x_l,x_r]$, which has as asymptotic expansion
around $\{x_l, x_r\}$ of the form
\begin{equation}
f(x_l-\delta) = f(x_l)+\lambda_l\delta^{k_l}+o(\delta^{k_l})
\text{~~and~~} 
f(x_r+\delta) = f(x_r)+\lambda_r\delta^{k_r}+o(\delta^{k_r})
\end{equation}
for $\delta>0$, some powers $k_l,k_r>1$, and constants
$\lambda_l,\lambda_r>0$.
The following result shows that if any algorithm significantly outperforms 
the local modulus of continuity on such a function, then it underperforms
the modulus on a nearby function.

\begin{proposition}
\label{prop:good-asymptotic-superefficiency}
Let $f$ be any convex function satisfying the asymptotic
expansion~\eqref{eqn:local-growth1} around its optimum.  Suppose that $A\in \A_T$ is any
algorithm that satisfies
\begin{equation}
\label{eq:superf}
\E_f\,\err(\hat x_A,f) \le \sqrt{\E_f\,\err(\hat x_A,f)^2} \le \delta_T\, \modcont_f\left(\frac{\sigma}{\sqrt{T}}\right),
\end{equation}
where $\delta_T < C_1$.  Define $g_{-1}(x) = f(x) - \epsilon_T x$ and
$g_1(x) = f(x) + \epsilon_T x$, where $\epsilon_T$ is given by $\epsilon_T  = \sqrt{\sfrac{\sigma^2\log\bigl(
    \frac{C_1}{\delta_T}\bigr)}{T}}$.
Then for some $g \in \{g_{-1}, g_1\}$, there exists $T_0$ such that $T \ge T_0$ implies
\begin{equation}
  \E_g\,\err(\hat x_A,g) \ge C\, \omega_g \left( \sqrt{\frac{\sigma^2\log \bigl(\sfrac{C_1}{\delta_T}\bigr)}{T}}
  \right)
\end{equation}
for some constant $C$ that only depends on $\pow = \pow_l \vee \pow_r$.  
\end{proposition}
A proof of this result is given in the appendix, where it is derived as a
consequence of a more general statement.
We remark that while condition \eqref{eq:superf} involves the squared error
$\sqrt{\E_f\,\err(\hat x_A,f)^2}$, we expect that the 
result holds with only the weaker inequality on the absolute error $\E_f\,\err(\hat x_A,f)$.

It follows from this proposition that if an algorithm $A$ significantly outperforms the local minimax complexity 
in the sense that \eqref{eq:superf} holds 
for some sequence $\delta_T\to 0$ with $\liminf_Te^T\delta_T=\infty$,
then there exists a sequence of convex functions $g_T$ with $\kappa(f,g_T)\to 0$, such that
\begin{equation}\label{eqn:local-growth1}
\liminf_{T\to\infty}\frac{\E_{g_T}\,\err(\hat x_A,g_T)}{\omega_{g_T}
\left(\sqrt{\sfrac{\sigma^2\log\bigl(\frac{C_1}{\delta_T}\bigr)}{T}}\right)}
> 0.
\end{equation}

This is analogous to the phenomenon of superefficiency in classical
parametric estimation problems, where outperforming
the asymptotically optimal rate given by the Fisher
information implies worse performance at some other point in the
parameter space. In this sense, $\omega_f$ can be viewed as 
a computational analogue of Fisher information in the setting of convex
optimization.  We note that superefficiency has also been
studied in nonparametric settings \citep{BrownLo96}, and a similar
result was shown by \citet{CaiLow15} for local minimax estimation of
convex functions.

\section{An adaptive optimization algorithm}\label{sec:algor}

In this section, we show that 
a simple stochastic binary search algorithm achieves the local minimax
complexity in the one-dimensional case.

The general idea of the algorithm is as follows.
Suppose that we are given a budget of $T$ queries to the oracle. 
We divide this budget into $T_0=\lfloor T/E\rfloor$ queries over each
of $E = \lfloor r\log T\rfloor$ many rounds,
where $r>0$ is a constant to be specified later. 
In each round, we query the oracle $T_0$ times for the derivative at the mid-point
of the current interval.  Estimating the derivative by averaging 
over the queries, we proceed to the left half of the interval
if the estimated sign is positive, and to the right half
of the interval of the estimated sign is negative.
The details are given in Algorithm \ref{alg:binary-search}.

\begin{algorithm}
\normalsize
\caption{Sign testing binary search}\label{alg:binary-search}
\begin{algorithmic}
\STATE Input: $T$, $r$.
\STATE Initialize: $(a_0,b_0)$, $E=\lfloor r\log T\rfloor$, $T_0=\lfloor T/E\rfloor$.
\FOR {$e=1,\dots,E$}
\STATE Query $x_e=(a_e+b_e)/2$ for $T_0$ times to get $Z_{t}^{(e)}$ for $t=1,\dots,T_0$.
\STATE Calculate the average $\bar Z_{T_0}^{(e)}=\frac{1}{T_0}\sum_{t=1}^{T_0}Z_t^{(e)}$.
\STATE If $\bar Z_{T_0}^{(e)}>0$, set $(a_{e+1},b_{e+1})=(a_e,x_e)$.
\STATE If $\bar Z_{T_0}^{(e)}\leq 0$, set $(a_{e+1},b_{e+1})=(x_e,b_e)$.
\ENDFOR
\STATE Output: $x_E$.
\end{algorithmic}
\end{algorithm}

We will show that this algorithm adapts to the local minimax complexity
up to a logarithmic factor. 
First, the following result shows that the algorithm gets us close to the ``flat set''
of the function.

\begin{proposition}\label{prop:algorithm}
For $\delta\in(0,1)$, let $C_\delta=\sigma\sqrt{2\log(E/\delta)}$. Define
\begin{equation}
\mc{I}_\delta=\left\{y\in\dom(f):|f'(y)|<\frac{C_\delta}{\sqrt{T_0}}\right\}.
\end{equation}
Suppose that $(a_0,b_0)\cap\mc{I}_\delta\neq \emptyset$.
Then
\begin{equation}
\text{dist}(x_E,\mc{I}_\delta)\leq2^{-E}(b_0-a_0)
\end{equation}
with probability at least $1-\delta$.
\end{proposition}

This proposition tells us that after $E$ rounds of bisection, we are
at most a distance $2^{-E}(b_0-a_0)$ from the flat set $\mc{I}_\delta$.
In terms of the distance to the minimum point, we have
\begin{align}
\inf_{x\in\mc{X}_f^*} |x_E-x|&\leq 2^{-E}(b_0-a_0)+\sup\Bigl\{\inf_{x\in\mc{X}_f^*}|x-y|:y\in\mc{I}_\delta\Bigr\}.
\end{align}
If the modulus of continuity satisfies the polynomial growth
condition, we then obtain the following.
\begin{corollary}\label{cor:algorithm}
Let $\alpha_0>0$. Suppose $\omega_f$ satisfies the polynomial growth
condition \eqref{eqn:polygrowth} with constant $\alpha\leq\alpha_0$.
Let $r=\half\alpha_0$. Then with probability at least $1-\delta$ and for large enough $T$,
\begin{equation}
\inf_{x\in\mc{X}_f^*}|x_E-x|\leq \tilde C\kern.1ex \omega_f\left(\frac{\sigma}{\sqrt{T}}\right)
\end{equation}
where the term $\tilde C$ hides a dependence on $\log T$ and $\log(1/\delta)$.
\end{corollary}
The proofs of these results are given in the appendix.

\subsection{Simulations showing adaptation to the benchmark}

We now demonstrate the performance of the stochastic binary search algorithm,
making a comparision to stochastic gradient descent.
For the stochastic gradient descent algorithm, we perform $T$ steps of update
\begin{equation}
x_{t+1}=x_t-\eta(t)\cdot \hat g(x_t)
\end{equation}
where $\eta(t)$ is a stepsize function, 
chosen as either $\eta(t)=\frac{1}{t}$ or $\eta(t)=\frac{1}{\sqrt{t}}$.
We first consider the following setup with symmetric functions $f$:
\begin{enumerate}
\item The function to optimize is
  $f_\pow(x)=\frac{1}{\pow}|x-x^*|^\pow$ for $\pow=\frac{3}{2},\,2$ or
  $3$.
\item The minimum point $x^*\sim\text{Unif}(-1,1)$ is selected
  uniformaly at random over the interval.
\item The oracle returns the derivative at the query point with additive
  $N(0,\sigma^2)$ noise, $\sigma=0.1$.
\item The optimization algorithms know \emph{a priori} that the
  minimum point is inside the interval $(-2,2)$. Therefore, the binary
  search starts with interval $(-2,2)$ and the stochastic gradient
  descent starts at $x_0\sim\text{Unif}(-2,2)$ and project the query
  points to the interval $(-2,2)$.
\item We carry out the simulation for values of $T$ on a logarithmic
  grid between 100 and 10{,}000. For each 
  setup, we average the error $|\hat x-x^*|$ over 1{,}000 runs.
\end{enumerate}
The simulation results are shown in the top 3 panels of Figure~\ref{fig:simulation}.
Several properties predicted by our theory are apparent from the simulations.
First, the risk curves for the stochastic binary search algorithm
parallel the gray curves.
This indicates that the optimal rate of convergence is achieved. Thus,
the stochastic binary search adapts to the curvature of different functions
and yields the optimal local minimax complexity, as given by our benchmark.
Second, the stochastic gradient descent algorithms with stepsize $1/t$ achieve the optimal rate when
$\pow=2$, but not when $\pow=3$; with stepsize $1/\sqrt{t}$ SGD gets close to the
optimal rate when $\pow=3$, but not when $\pow=2$.
Neither leads to the faster rate when $\pow=\frac{3}{2}$. This is as
expected, since the stepsize needs to be adapted to the curvature at
the optimum in order to achieve the optimal rate.

Next, we consider a set of asymmetric functions. Using the same setup
as in the symmetric case, we consider the functions of the form 
$f(x)=\frac{1}{k_l}|x-x^*|^{k_l}I(x-x^*\leq 0)+\frac{1}{k_r}|x-x^*|^{k_r}I(x-x^*>0)$,
for exponent pairs $(\pow_1,\pow_2)$ chosen to be $(\frac{3}{2},2)$, $(\frac{3}{2},3)$ and $(2,3)$.
The simulation results are shown in the bottom three panels of Figure~\ref{fig:simulation}.
We observe that the stochastic binary search once again achieves 
the optimal rate, which is determined by the flatter side of the function,
that is, the larger of $\pow_l$ and $\pow_r$.

\begin{figure}
\begin{center}
\input{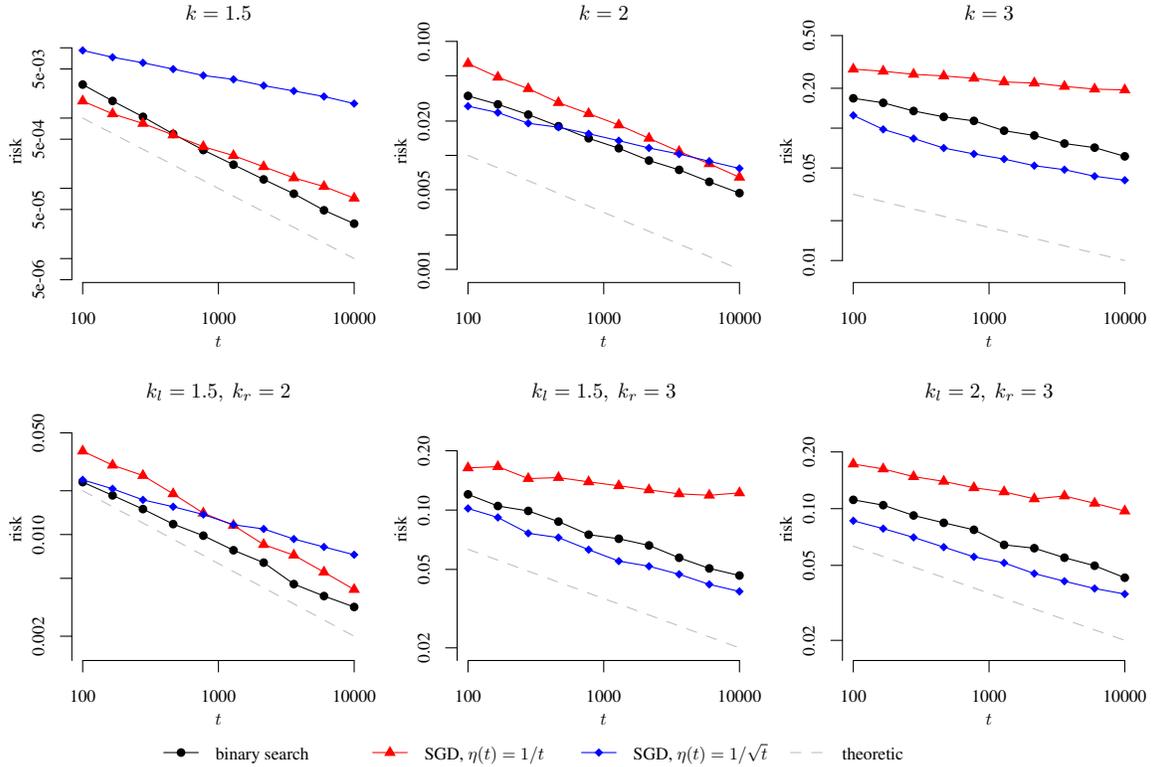}
\vskip-20pt
\caption{Simulation results: Averaged risk versus number of queries $T$.
The black curves correspond to the risk of  the stochastic binary search algorithm.
The red and blue curves are for the stochastic gradient descent methods, red for 
stepsize $1/t$ and blue for $1/\sqrt{t}$.
The dashed gray lines indicate the optimal convergence rate.
Note that the plots are on a log-log scale.
The plots on the top panels are for the symmetric cases $f(x)=\frac{1}{\pow}|x-x^*|^\pow$;
the lower plots are for the asymmetric cases.}\label{fig:simulation}
\end{center}
\end{figure}

\section{Related work and future directions}

In related recent work, \citet{ramdas2013optimal} study minimax
complexity for the class of Lipschitz convex functions that satisfy
$f(x)-f(x_f^*)\geq\frac{\lambda}{2}\|x-x_f^*\|^k$. They show that the
minimax complexity under the function value error is of the order
$T^{-\frac{k}{2(k-1)}}$.
\citet{iouditski2014primal} also consider minimax complexity for the
class of $k$-uniformly convex functions for $k>2$.  They give an
adaptive algorithm based on stochastic gradient descent that achieves
the minimax complexity up to a logarithmic factor. Connections with
active learning are developed in \cite{ramdas2013algorithmic}, with related ideas
appearing in \cite{castro2008minimax}.  Adaptivity in this line of
work corresponds to the standard notion in statistical estimation,
which seeks to adapt to a large subclass of a parameter space.  In
contrast, the results in the current paper quantify the difficulty of
stochastic convex optimization at a much finer scale, as the benchmark
is determined by the specific function to be optimized.

The stochastic binary search algorithm presented in Section
\ref{sec:algor}, despite being adaptive, has a few drawbacks.  It
requires the modulus of continuity of the function to satisfy
polynomial growth, with a parameter $\alpha$ bounded away from 0.
This rules out cases such as $f(x)=|x|$, which should have an error
that decays exponentially in $T$; it is of interest to handle this
case as well.  It would also be of interest to construct adaptive
optimization procedures tuned to a fixed numerical precision. Such
procedures should have different running times depending on the
hardness of the problem.  Progress on both problems has been made, and
will be reported elsewhere.

Another challenge is to remove the logarithmic factors appearing in
the binary search algorithm developed in Section~\ref{sec:algor}.  In
one dimension, stochastic convex optimization is intimately related to
a noisy root finding problem for a monotone function taking values in
$[-a,a]$ for some $a > 0.$ \citet{karp2007noisy} study optimal
algorithms for such root finding problems in a discrete setting. A
binary search algorithm that allows backtracking is proposed, which
saves log factors in the running time. It would be interesting to
study the use of such techniques in our setting.

Other areas that warrant study involve the dependence on dimension.
The scaling with dimension of the local minimax complexity and modulus
of continuity is not fully revealed by the current analysis. Moreover,
the superefficiency result and the adaptive algorithm presented here
are only for the one-dimensional case.  We note that a form of
adaptive stochastic gradient algorithm for the class of uniformly
convex functions in general, fixed dimension is developed in
\cite{iouditski2014primal}.  

Finally, a more open-ended direction is
to consider larger classes of stochastic optimization problems.  For
instance, minimax results are known for functions of the form
$f(x):=\mathbb E\, F(x;\xi)$ where $\xi$ is a random variable and
$x\mapsto F(x;\xi)$ is convex for any $\xi$, when $f$ is twice
continuously differentiable around the minimum point with positive
definite Hessian. However, the role of the local geometry is not well
understood. It would be interesting to further develop the local
complexity techniques introduced in the current paper, to gain insight
into the geometric structure of more general stochastic optimization
problems.

\begin{appendix}
\section{Proof of Theorem \ref{thm:localminimax}}

\subsection{Lower bound}
For a function $f \in \mc{F}$, let $P_f$ denote
the distribution of stochastic gradients observable by an estimation
scheme $\what{x}$, and let $P_f^T$ denote the distribution of
$T$ sequentially queried stochastic gradients for $f$.
We define the pairwise minimax
risk for optimization of a pair of function $f$ and $g$ by
\begin{equation}
  \label{eqn:pairwise-complexity}
  R_T(f, g) \defeq
  \inf_{A\in\mathcal A_T} \max\left\{
  \E_f\, \err(\what{x}_A, f),
  \E_g\, \err(\what{x}_A, g)
  \right\},
\end{equation}
and the local minimax lower bound can be written as
\begin{equation}
  \label{eqn:local-complexity}
  R_T(f;\F) \defeq \sup_{g \in \mc{F}} R_T(f, g).
\end{equation}

Let us show how the modulus of continuity 
gives a lower bound.
We first state a lemma.
\begin{lemma}
  \label{lemma:le-cam}
  Let $f, g$ be arbitrary convex functions and $\fndist$ satisfy the
  exclusion inequality~\eqref{eqn:exclusion-inequality}. Then
  \begin{equation}
    R_T(f, g)
    \ge \frac{\fndist(f, g)}{4}
    \left(1 - \tvnorm{P_f^T - P_g^T}\right).
  \end{equation}
\end{lemma}
\begin{proof}
  Temporarily hiding the number of iterations $T$ for simplicity,
  we have by Markov's inequality that
  \begin{align}
    \max&\left\{\E_f\,\err(\what{x}_A, f),
    \E_g\,\err(\what{x}_A, g)\right\}\\
    & \ge \half\fndist(f, g) \max\left\{
    P_f(\err(\what{x}_A, f) \ge \half\fndist(f, g)),
    P_g(\err(\what{x}_A, g) \ge \half\fndist(f, g))
    \right\}.
  \end{align}
  Now, we apply an essentially standard reduction of estimation to
  testing, because we have
  \begin{align}
    \lefteqn{2 \max\left\{
      P_f(\err(\what{x}_A, f) \ge \half\fndist(f, g)),
      P_g(\err(\what{x}_A, g) \ge \half\fndist(f, g))
      \right\}} \\
    & \qquad \ge 
    P_f(\err(\what{x}_A, f) \ge \half\fndist(f, g))
    + P_g(\err(\what{x}_A, g) \ge \half\fndist(f, g)) \\
    & \qquad = 1 - P_f(\err(\what{x}_A, f) < \half\fndist(f, g))
    + P_g(\err(\what{x}_A, g) \ge \half\fndist(f, g)) \\
    & \qquad \ge 1 - P_f(\err(\what{x}_A, g) \ge \half\fndist(f, g))
    + P_g(\err(\what{x}_A, g) \ge \half\fndist(f, g)),
  \end{align}
  where in the last line we have used the exclusion inequality
  to see that $\err(\what{x}_A, f) < \half\fndist(f, g)$ implies
  $\err(\what{x}_A, g) \ge \half\fndist(f, g)$ so that
  \begin{equation}
    P_f(\err(\what{x}_A, f) < \half\fndist(f, g))
    \le 
    P_f(\err(\what{x}_A, g) \ge \half\fndist(f, g)).
  \end{equation}
  Thus, we find that
  \begin{equation}
    \frac{4}{\fndist(f, g)} \max\left\{\E_f\,\err(\what{x}_A, f),
    \E_g\,\err(\what{x}_A, g)\right\}
    \ge \inf_S \left\{1 - P_f^T(S) + P_g^T(S)\right\}
    = 1 - \tvnorm{P_f^T - P_g^T},
  \end{equation}
  which yields the lemma.
\end{proof}

Now we can prove a minimax lower bound. Let $Y_i$ be the $i$th observed
gradient, where $P_f(Y_i \mid Y_{1:i-1})$ denotes the conditional
distribution of $Y_i$ under the oracle for function $f$. We have by
the chain rule that
\begin{align}
   \dkl{P_f^T}{P_g^T}
  = \sum_{i = 1}^T \E_{f}\left[\dkl{P_f(Y_i \mid Y_{1:i-1})}{
      P_g(Y_i \mid Y_{1:i-1})}\right].
\end{align}
It is no loss of generality to assume that the $i$th gradient query point
$x_i$ is measureable with respect to $Y_{1:i-1}$ (this follows because if a
randomized algorithm does well in expectation, there is at least one
realization of its randomness that has small risk, so we can just take that
realization and assume the procedure is deterministic). Using that we have a
Gaussian oracle, we have
\begin{align}
  \dkl{P_f(Y_i \mid Y_{1:i-1})}{
    P_g(Y_i \mid Y_{1:i-1})}
  & = \dkl{\normal(f'(x_i), \sigma^2 I_{d \times d})}{\normal(g'(x_i),
    \sigma^2 I_{d \times d})} \\
  & = \frac{1}{2\sigma^2}
  \norm{f'(x_i) - g'(x_i)}^2
  \le \frac{1}{2\sigma^2} \fnmetric(f, g)^2.
\end{align}
Noting the not completely standard upper bound
\begin{equation}
  \tvnorm{P_f^T - P_g^T}
  \le 1 - \exp\left(-\half \dkl{P_f^T}{P_g^T}\right)
\end{equation}
on the variation distance (see \citet[Lemma 2.6]{Tsybakov09}), we also have
by Lemma~\ref{lemma:le-cam} that
\begin{equation}
  R_T(f, g) \ge \frac{\fndist(f, g)}{4}
  \exp\left(-\frac{T}{4\sigma^2} \fnmetric(f, g)^2\right).
  \label{eqn:exponential-risk-bound}
\end{equation}

Consider the collection of functions
\begin{equation}
  \mc{F}_T \defeq \left\{g \in \mc{F} :
  \fnmetric(f, g)^2 \le \frac{\sigma^2}{T}
  \right\}.
\end{equation}
Certainly this collection is non-empty (it includes $f$).
For any $\epsilon > 0$, there must exist some $g \in \mc{F}_T$ such that
$d(f, g) \ge (1 - \epsilon) \modcont_f(1 / \sqrt{T})$. Let
$g_T$ denote such a $g$. Then we have
\begin{equation}
  R_T(f) \ge R_T(f, g_T)
  \ge \frac{\fndist(f, g_T)}{4} e^{-\frac{1}{4}}
  \ge  \frac{1 - \epsilon}{4}e^{-\frac{1}{4}}
  \modcont_f\left(\frac{\sigma}{\sqrt{T}}
  \right).
\end{equation}
We have 
  \begin{equation}
    R_T(f) \ge
    \frac{1}{4e^{1/4}} \modcont_f\left(\frac{\sigma}{\sqrt{T}}\right)
    \ge \frac{3}{16} \modcont_f\left(\frac{\sigma}{\sqrt{T}}\right).
  \end{equation}

\subsection{Upper bound}

Suppose that we have two functions $f_{-1},f_1\in\mathcal F$. Let 
\begin{equation}\label{eqn:xdag}
x^\dag = \argmax_{x\in \C}\left\{\|f_{-1}'(x)-f_{1}'(x)\|\right\}
\end{equation}
be the point at which the two functions differ the most in terms of the subgradients. 
Let $\theta\in\{-1,1\}$ be the parameter.
Consider an algorithm that queries the oracle with $x^\dag$ for $T$ times. 
Let $Z_t$ be the response from the oracle at time $t$.
Let
\begin{equation}
W=\frac{1}{\sqrt{T}}\sum_{t=1}^TZ_t-\frac{\sqrt{T}}{2}(f'_1(x^\dag)+f'_{-1}(x^\dag))
\end{equation}
With the normality assumption on the noise, we have
\begin{equation}
W\sim N(\theta\gamma_T,\sigma^2I)
\end{equation}
where
\begin{equation}
\gamma_T = \frac{\sqrt{T}}{2}\left(f'_1(x^\dag)-f'_{-1}(x^\dag)\right).
\end{equation}
Then we construct
\begin{equation}
\overline W = \|\gamma_T\|^{-1}\gamma_T^\T W\sim N(\theta\|\gamma_T\|,\sigma^2),
\end{equation}
which is a sufficient statistic for the problem of estimating $\theta$.
Based on $\overline W$ we can obtain an estimate $\hat\theta$ of $\theta$, 
and let the output of our algorithm be
\begin{equation}
\hat x_T = \frac{x_1^*+x_{-1}^*}{2}+\hat\theta\frac{x_1^*-x_{-1}^*}{2}
\end{equation}
where $x_1^*\in\mc{X}_{f_1}^*$ and $x_{-1}^*\in\mc{X}_{f_{-1}}^*$ satisfy
$\|x_1-x_{-1}\|=\inf_{x\in\mc{X}_{f_1}^*}\inf_{y\in\mc{X}_{f_{-1}}^*}\|x-y\|$.
It then follows 
\begin{align}
\inf_{A\in\mathcal A_T}\max_{\theta=\pm 1}\mathbb E_\theta\|\hat x_A-x_\theta^*\|&\leq\max_{\theta=\pm 1}\mathbb E_\theta\|\hat x_T-x_\theta^*\|\\
&\leq \frac{1}{2}\|x_{1}^*-x_{-1}^*\|\inf_{\hat\theta}\sup_{\theta=\pm1}\mathbb E_\theta|\hat\theta-\theta|\\
&=\frac{1}{2}\|x_{1}^*-x_{-1}^*\|\|\gamma_T\|^{-1}\lambda (\|\gamma_T\|,\sigma)\label{eqn:twopoint}
\end{align}
where $\lambda(\tau,\sigma)=\inf_{\hat\mu}\sup_{\mu=\pm\tau}\mathbb E_\mu|\hat\mu-\mu|$ is the minimax ($\ell_1$) risk of estimating the mean of $Z\sim N(\tau,\sigma^2)$ for the class $\mu\in\{-\tau,\tau\}$.

Now take $f_{-1}=f$ and $f_1=g$. Note that
$\|\gamma_T\| = \frac{\sqrt{T}}{2}\kappa(f,g)$.
From \eqref{eqn:twopoint} we have
\begin{align}
R_T(f;\mathcal F)&=\sup_{g\in\mathcal F}\inf_{A\in\mathcal A_T}\max\left\{\mathbb E_f\|\hat x_T-x_f^*\|,\mathbb E_g\|\hat x_T-x_g^*\|\right\}\\
&\leq\frac{1}{2}\sup_{\|\gamma_T\|}\sup_{g\in\mathcal F:\kappa(f,g)=\frac{2\|\gamma_T\|}{\sqrt{T}}}\|x_f^*-x_g^*\|\|\gamma_T\|^{-1}\lambda(\|\gamma_T\|,\sigma)\\
&\leq\frac{1}{2}\sup_\tau\omega_f\left(\frac{2\tau}{\sqrt{T}}\right)\tau^{-1}\lambda(\tau,\sigma).
\end{align}
We have the following bound derived from \cite{donoho1994statistical}
\begin{equation}
\lambda(\tau,\sigma)\leq \tau\exp\left(-\frac{\tau^2}{4\sigma^2}\right),
\end{equation}
which yields
\begin{equation}
R_T(f;\mathcal F)\leq\frac{1}{2}\sup_\tau\omega\left(\frac{2\tau}{\sqrt{T}}\right)\exp\left(-\frac{\tau^2}{4\sigma^2}\right).
\end{equation}
To upper bound the last quantity, we write
\begin{align}
\sup_\tau\omega\left(\frac{2\tau}{\sqrt{T}}\right)\exp\left(-\frac{\tau^2}{4\sigma^2}\right)\leq\max\Bigg\{&\sup_{\tau\leq r}\psi(\tau),
\sup_{r<\tau\leq\half\epsilon_0\sqrt{T}}\psi(\tau),
\sup_{\tau>\half\epsilon_0\sqrt{T}}\psi(\tau)\Bigg\}
\end{align}
for some $r>0$, where $\psi(\tau)=\omega\left(\frac{2\tau}{\sqrt{T}}\right)\exp\left(-\frac{\tau^2}{4\sigma^2}\right)$.
We bound the three terms separately by
\begin{equation}
\sup_{\tau\leq r}\omega\left(\frac{2\tau}{\sqrt{T}}\right)\exp\left(-\frac{\tau^2}{4\sigma^2}\right)\leq \omega\left(\frac{2r}{\sqrt{T}}\right),
\end{equation}
and
\begin{align}
&\sup_{r<\tau\leq\half\epsilon_0\sqrt{T}}\omega\left(\frac{2\tau}{\sqrt{T}}\right)\exp\left(-\frac{\tau^2}{4\sigma^2}\right)\\
&=\sup_{s\geq 1\,\&\,\frac{2sr}{\sqrt{T}}\leq\epsilon_0}\omega\left(\frac{2sr}{\sqrt{T}}\right)\exp\left(-\frac{s^2r^2}{4\sigma^2}\right)\\
&\leq \sup_{s\geq 1}s^\alpha\omega\left(\frac{2r}{\sqrt{T}}\right)\exp\left(-\frac{s^2r^2}{4\sigma^2}\right)\\
&\leq \left(\frac{\sqrt{2\alpha}\sigma}{r}\right)^\alpha \omega\left(\frac{2r}{\sqrt{T}}\right)
\end{align}
since $\omega_f$ satisfies $\omega_f(c\epsilon)\leq c^\alpha\omega_f(\epsilon)$ for $c>1$, $c\epsilon\leq\epsilon_0$ and some $\alpha>0$, and
\begin{align}
\sup_{\tau>\frac{1}{2}\epsilon_0\sqrt{T}}\omega\left(\frac{2\tau}{\sqrt{T}}\right)\exp\left(-\frac{\tau^2}{4\sigma^2}\right)
&\leq\text{diam}(\C)\exp\left(-\frac{\epsilon_0^2T}{16\sigma^2}\right)
\end{align}
Setting $r=\sigma/2$ and noting that $\omega_f(\epsilon)\geq \epsilon^{\alpha}\frac{\omega_f(\epsilon_0)}{\epsilon_0^\alpha}$, we have that there exists $T_0>0$ such that for all $T\geq T_0$
\begin{equation}
R_T(f;\mathcal F)\leq C\omega_f\left(\frac{\sigma}{\sqrt{T}}\right)
\end{equation}
where $C=\frac{1}{2}\max\{1,\left(8\alpha\right)^{\frac{\alpha}{2}}\}$.

\section{Proofs for superefficiency results}

We begin by recalling the following results about properties of the
subdifferential of a convex function $f$ and its Fenchel conjugate
\begin{equation}
  f^*(y) \defeq \sup_x \left\{y^Tx - f(x) \right\},
\end{equation}
including duality between the subdifferential sets
$\partial f$ and $\partial f^*$, increasing gradients, and continuous
differentiability.

\begin{lemma}[\citet{HiriartUrrutyLe93ab}]
  \label{lemma:subgradients}
  Let $f$ be a closed convex function. Then
  \begin{equation}
    x \in \partial f^*(y)
    ~~ \mbox{if and only if} ~~
    y \in \partial f(x).
  \end{equation}
  Additionally, subgradient sets are increasing in the sense that
  \begin{equation}
    s_1 \in \partial f(x_1) ~ \mbox{and} ~
    s_2 \in \partial f(x_2)
    ~~ \mbox{implies} ~~
    \<s_1 - s_2, x_1 - x_2\> \ge 0.
  \end{equation}
  Lastly, if $f : \R \to \R$ is strictly convex on an interval
  $[x_l, x_r]$, then $f^*$ is continuously differentiable on
  the interval $[\inf \{s : s \in \partial f(x_l)\},
    \sup \{s : s \in \partial f(x_r) \}]$.
\end{lemma}

\subsection{Moduli of continuity}

\begin{lemma}
  \label{lemma:modcont}
  Let $f : \R \to \R$ be a subdifferentiable convex function. Define
  $f_\epsilon(x) = f(x) + \epsilon x$. Then
  \begin{equation}
    \argmin_x f_\epsilon(x) = \partial f^*(-\epsilon)
  \end{equation}
  Moreover,
\begin{gather}
     \dist(\partial f^*(0), \partial f^*(\epsilon))
    \vee \dist(\partial f^*(0), \partial f^*(-\epsilon)) 
     \le \modcont_f(\epsilon) \\
     \modcont_f(\epsilon)  \le
    \sup_x \{\dist(x, \partial f^*(0)) : x \in \partial f^*(\epsilon)\}
    \vee
    \sup_x \{\dist(x, \partial f^*(0)) : x \in \partial
    f^*(-\epsilon)\}
  \end{gather}
  In particular, if $x_0 = \argmin_x f(x)$ is unique
  and $f$ is strictly convex in a neighborhood of $x_0$, then
  there exists an $\epsilon_0 > 0$ such that $\epsilon \le \epsilon_0$ implies
  that
  \begin{equation}
    \modcont_f(\epsilon) = \max\left\{|{f^*}'(\epsilon) - x_0|,
    |{f^*}'(-\epsilon) - x_0|\right\}.
  \end{equation}
\end{lemma}
\begin{proof}
  Let $x_0 \in \argmin_x f(x)$. Using Lemma~\ref{lemma:subgradients},
  it is clear that $\argmin_x f(x) = \partial f^*(0)$, and more generally, that
  \begin{equation}
    \label{eqn:argmins}
    \partial f^*(y)
    = \argmax_x \left\{y^T x - f(x) \right\}
    = \argmin_x \left\{f(x) - y^T x\right\}.
  \end{equation}

  We begin by providing the lower bound on $\modcont_f$.
  For $\epsilon > 0$, define the function
  $f_\epsilon(x) = f(x) + \epsilon x$. Then
  certainly $\fnmetric(f, f_\epsilon) \le \epsilon$. Moreover,
  we have
  \begin{equation}
    f_\epsilon^*(y)
    = \sup_x \{yx - f(x) - \epsilon x\}
    = \sup_x \{(y - \epsilon) x - f(x)\}
    = f^*(y - \epsilon),
  \end{equation}
  so that $\argmin_x f_\epsilon(x) = \partial f^*(-\epsilon)$.
  Noting that $x_0 \in \partial f^*(0)$ and that
  subgradients are increasing by Lemma~\ref{lemma:subgradients}, we have that
  \begin{equation}
    \argmin_x f_\epsilon(x) = \partial f^*(-\epsilon) \le \partial f^*(0)
    = \argmin_x f(x).
  \end{equation}
  That is, we have $\sup\{x_\epsilon \in \argmin_x f_\epsilon(x)\}
  \le \inf\{x_0 \in \argmin_x f(x)\}$
  and
  \begin{equation}
    \modcont_f(\epsilon)
    \ge \inf\left\{|s_\epsilon - s_0| : s_\epsilon \in \partial f^*(-\epsilon),
    s_0 \in \partial f^*(0)\right\}.
  \end{equation}
  An identical argument with $f_{-\epsilon}$ gives the lower bound.

  For the upper bound on the modulus of continuity, we note that
  if $g$ is a convex function with $\fnmetric(f, g) \le \epsilon$, and
  $x_g \in \argmin_x g(x)$, then
  there must be some $s \in \partial f(x_g)$ with
  $\epsilon \ge s \ge -\epsilon$,
  because $0 \in \partial g(x_g)$, where we have used the
  definition of the Hausdorff distance.
  Now, for this particular $s$, by Lemma~\ref{lemma:subgradients}
  we have that
  \begin{equation}
    x_g \in \partial f^*(s).
  \end{equation}
  Again using the increasing behavior of subgradients, we obtain that
  \begin{equation}
    \inf \partial f^*(-\epsilon)
    \le x_g \le \sup \partial f^*(\epsilon),
  \end{equation}
  which gives the claimed upper bound in the lemma once we recognize
  that $x_0 \in \partial f^*(0)$, and the definition
  of distance for $\modcont_f$ is $\fndist(f, g) = \inf \{|x_0 - x_g\opt|
  : x_0 \in \argmin_x f(x), x_g\opt \in \argmin_x g(x)\}$.

  The final result, with the uniqueness, is an immediate consequence of the
  differentiability properties in Lemma~\ref{lemma:subgradients}.
\end{proof}


Now we calculate bounds for a few
example moduli of contiuity using Lemma~\ref{lemma:modcont}. Roughly, we
focus on non-pathological convex functions to allow us to give explicit
calculations. Let $f : \R \to \R$ be a convex function satisfying $\partial
f^*(0) = \argmin_x f(x) = [x_l, x_r]$. In addition, assume that
for $\delta > 0$, we have for some powers
$\pow_l, \pow_r \ge 1$ and constants $\lambda_l > 0$ and $\lambda_r > 0$
that
\begin{equation}
  \label{eqn:local-growth}
  f(x_l - \delta) = f(x_l) + \lambda_l \delta^{\pow_l}
  + o(\delta^{\pow_l})
  ~~ \mbox{and} ~~
  f(x_r + \delta) = f(x_r) + \lambda_r \delta^{\pow_r}
  + o(\delta^{\pow_r}).
\end{equation}
That is, in a neighborhood of the optimal region, the function $f$
grows like a polynomial. The condition~\eqref{eqn:local-growth} is
not too restrictive, but does rule out functions such as
$f(x) = e^{-\frac{1}{x^2}}$.
\begin{lemma}
  \label{lemma:growth-f*}
  Let $f$ satisfy the condition~\eqref{eqn:local-growth}.
  For any $c > 1$,
  there exists some $\epsilon_0 > 0$ such that for
  $\epsilon \in (0, \epsilon_0)$
  \begin{subequations}
    \begin{equation}
      \label{eqn:right-side-bound}
      x_r + \left(\frac{\epsilon}{C \lambda_r \pow_r}\right)^\frac{1}{
        \pow_r - 1}
      \le \inf \partial f^*(\epsilon)
      \le \sup \partial f^*(\epsilon)
      \le x_r + \left(\frac{C \epsilon}{\lambda_r}\right)^\frac{1}{\pow_r - 1}
    \end{equation}
    and
    \begin{equation}
      \label{eqn:left-side-bound}
      x_l - \left(\frac{\epsilon}{C \lambda_l  \pow_l}\right)^\frac{1}{
        \pow_l - 1}
      \ge \sup \partial f^*(-\epsilon)
      \ge \inf \partial f^*(-\epsilon)
      \ge x_l - \left(\frac{C \epsilon}{\lambda_l}\right)^\frac{1}{\pow_l - 1}.
    \end{equation}
  \end{subequations}
  Moreover, setting $\pow = \max\{\pow_r, \pow_l\}$ and
  letting
  \begin{equation}
    \lambda = \begin{cases}
      \lambda_l & \mbox{if} ~ \pow_l > \pow_r, \\
      \lambda_r & \mbox{if} ~ \pow_r > \pow_l, \\
      \max\{\lambda_r, \lambda_l\} & \mbox{otherwise},
    \end{cases}
  \end{equation}
  we have for all $\epsilon \in (0, \epsilon_0)$ that
  \begin{equation}
    \left(\frac{\epsilon}{C \lambda \pow}\right)^\frac{1}{\pow - 1}
    \le \modcont_f(\epsilon)
    \le \left(\frac{C \epsilon}{\lambda}\right)^\frac{1}{\pow - 1}.
  \end{equation}
\end{lemma}
\begin{proof}
  We focus on the right side bound~\eqref{eqn:right-side-bound},
  as the proof of the left bound~\eqref{eqn:left-side-bound} is
  similar. We also let the constant be $c = 2$ for simplicity.

  For notational simplicity, let $\lambda = \lambda_r$ and $\pow = \pow_r$.
  By the fact that subgradients are increasing, we have for any $\delta > 0$
  that
  \begin{equation}
    \inf \partial f(x_r + \delta) \ge
    \frac{f(x_r + \delta) - f(x_r)}{\delta}
    = \frac{\lambda \delta^\pow + o(\delta^\pow)}{\delta}
    = \lambda (1 - o_\delta(1)) \delta^{\pow - 1}
    \label{eqn:lower-gradient-bound}
  \end{equation}
  as $\delta \downarrow 0$. Similarly, $\delta > 0$ we have
  \begin{align}
    \sup \partial f(x_r + \delta)
    & \le \frac{f(x_r + 2 \delta) - f(x_r + \delta)}{
      \delta}
    = \frac{\lambda(2\delta)^\pow - \lambda \delta^\pow
    + o(\delta^\pow)}{\delta} \nonumber \\
    & = \frac{\lambda \pow \delta^{\pow - 1} \delta
      + o(\delta^\pow)}{\delta}
    = (1 + o_\delta(1)) \lambda \pow \delta^{\pow - 1}.
    \label{eqn:upper-gradient-bound}
  \end{align}
  Combining inequalities~\eqref{eqn:lower-gradient-bound}
  and~\eqref{eqn:upper-gradient-bound},
  we thus see that there exists some $\delta_0 > 0$ such that for
  $\delta \in (0, \delta_0)$ we have
  \begin{equation}
    \frac{\lambda}{2} \delta^{\pow - 1}
    \le \inf \partial f(x_r + \delta)
    \le \sup \partial f(x_r + \delta)
    \le 2 \lambda \pow \delta^{\pow - 1}.
  \end{equation}
  Noting that $x_r + \delta \in \partial f^*(\epsilon)$ if and only if
  $\epsilon \in \partial f(x_r + \delta)$ by standard subgradient calculus
  (recall Lemma~\ref{lemma:subgradients}), we solve for $\epsilon =
  \frac{\lambda}{2} \delta^{\pow - 1}$ and $\epsilon = 2 \lambda \pow
  \delta^{\pow - 1}$ to attain inequality~\eqref{eqn:right-side-bound}. The
  bound~\eqref{eqn:left-side-bound}
  is similar.
\end{proof}

Lemma~\ref{lemma:growth-f*} shows that, as $\epsilon \to 0$,
we have $\modcont_f(\epsilon) \asymp \epsilon^\frac{1}{\pow - 1}$,
where $\pow = \max\{\pow_r, \pow_l\}$.
Finally, we show a type of continuity property with the
modulus of continuity.
\begin{lemma}
  \label{lemma:modconts-growth}
  Assume that $f$ has expansion~\eqref{eqn:local-growth}, and that
  either (i) $\pow_r > \pow_l$ or (ii) $\pow_r \ge \pow_l$ and
  $\lambda_r \ge \lambda_l$. Define
  $g(x) = f(x) - \epsilon x$. Then for any constants $c < 1 < C$,
  we have
  \begin{equation}
    \modcont_g(c \epsilon) \le (2C)^\frac{1}{\pow_r - 1}
    \left(\frac{\epsilon}{\lambda_r}
    \right)^\frac{1}{\pow_r - 1}
    \le (2C^2)^\frac{1}{\pow_r - 1} e \, \modcont_f(\epsilon)
  \end{equation}
  for all $\epsilon$ suitably close to $0$.
\end{lemma}
\begin{proof}
  We know by the increasing properties of the subgradient set
  and Lemma~\ref{lemma:modcont} that for any $c < 1$
  \begin{align}
    \modcont_g\left(c \epsilon\right)
    & \le \max\{\dist(\partial g^*(\epsilon), \partial g^*(0)),
    \dist(\partial g^*(-\epsilon), \partial g^*(0))\} \\
    & = \max\{\dist(\partial f^*(2 \epsilon), \partial f^*(\epsilon)),
    \dist(\partial f^*(0), \partial f^*(\epsilon))\},
  \end{align}
  where we have used that $g^*(y) = \sup_x \{(y + \epsilon) x - f(x)\} =
  f^*(y + \epsilon)$. For small enough $\epsilon > 0$, we have by
  Lemma~\ref{lemma:growth-f*} that
  \begin{equation}
    \sup \partial f^*(2 \epsilon)
    \le \left(\frac{2 C \epsilon}{\lambda_r}\right)^\frac{1}{\pow_r - 1},
  \end{equation}
  which gives the first inequality.

  For the second inequality, we use that $\modcont_f(\epsilon) \ge (\epsilon
  / (C \lambda_r \pow_r))^\frac{1}{\pow_r - 1}$ to obtain
  \begin{equation}
    \left(\frac{2C \epsilon}{\lambda_r}\right)^\frac{1}{\pow_r - 1}
    = \pow_r^\frac{1}{\pow_r - 1}
    (2C^2)^\frac{1}{\pow_r - 1}
    \left(\frac{\epsilon}{C \lambda_r \pow_r}\right)^\frac{1}{\pow_r - 1}
    \le \pow_r^\frac{1}{\pow_r - 1}
    (2C^2)^\frac{1}{\pow_r - 1}\modcont_f(\epsilon)
    \le e (2C^2)^\frac{1}{\pow_r - 1}\modcont_f(\epsilon)
  \end{equation}
  as desired.
\end{proof}

\subsection{Superefficiency}

For distributions $P_0$ and $P_1$ define the
$\chi$-divergence by
\begin{align}
  \dchii{P_1}{P_0}
  &\defeq \int \left(\frac{dP_1}{dP_0} - 1\right) dP_1
  = \int \left(\frac{dP_1}{dP_0}\right) dP_1 - 1.
\end{align}
The following lemma, which is a stronger version of a result
due to \citet{BrownLo96}, gives
a result on superefficiency.
\begin{lemma}
\label{lemma:brown-low-stronger}
  Let $\what{x}$ be any function of a
  sample $\statrv$, and let $X_0$ and $X_1$ be compact convex sets
  (associated with distributions $P_0$ and $P_1$). Let $\dist(x, X) =
  \inf_{y \in X} |y - x|$ and $\dist(X_0, X_1) = \inf_{x_0 \in X_0}
  \dist(x_0, X_1)$.  Then
  \begin{align}
    \E_{P_1}[\dist(\what{x}, X_1)]
    & \ge \hinge{\dist(X_0, X_1)
      - \sqrt{\E_{P_0}[\dist(\what{x}, X_0)^2]
        (\dchii{P_1}{P_0} + 1)}} \\
    & \ge \dist(X_0, X_1)
    \hinge{1 - \frac{\sqrt{\E_{P_0}[\dist(\what{x}, X_0)^2]
          (\dchii{P_1}{P_0} + 1)}}{\dist(X_0, X_1)}}.
  \end{align}
\end{lemma}
\begin{proof}
  We have
  \begin{align*}
    \E_{P_1}[\dist(\what{x}, X_1)]
    & \stackrel{(i)}{\geq} \dist(X_0, X_1) - \E_{P_1}[\dist(\what{x},X_0)]\\
    & \stackrel{(ii)}{\geq}
    \dist(X_0, X_1) - \sqrt{\E_{P_0}[\dist(\what{x},X_0)^2]\cdot \int\left(\frac{dP_1}{dP_0}\right)dP_1}\\
    &=\dist(X_0, X_1) - \sqrt{\E_{P_0}[\dist(\what{x}, X_0)^2](\dchii{P_1}{P_0} + 1)}
  \end{align*}
  where inequality (i) uses the triangle inequality
  and inequality (ii) uses Cauchy-Schwarz.
\end{proof}

We now present two lemmas on $\chi$-divergence that will be useful.
The first is a standard algebraic calculation.
\begin{lemma}
  Let $P_0$ and $P_1$ be normal distributions with means $\mu_0$ and $\mu_1$,
  respectively, and variances $\stddev^2$. Then
  \begin{align}
    \dchii{P_0}{P_1} &= \dchii{P_1}{P_0}
    = \exp\left(\frac{(\mu_0 - \mu_1)^2}{\stddev^2}\right) - 1.
  \end{align}
\end{lemma}
\noindent
For the second lemma, we assume that $\what{x}$ is constructed
based on noisy subgradient information from a subgradient oracle, which
upon being queried at a point $x$, returns
\begin{equation}
  \label{eqn:subgradient-oracle}
  f'(x) + \noise,
  ~~ \mbox{where}~ \noise \simiid \normal(0, \stddev^2)
  ~~ \mbox{and} ~~
  f'(x) = \argmin_{s \in \partial f(x)} \{|s|\}.
\end{equation}
The latter condition simply specifies the subgradient the oracle chooses;
any specified choice of subgradient is sufficient. Because $\partial f(x)$
is a closed convex set for any $x$, we see that if $f$ and $g$ are convex
functions with $\fnmetric(f, g) \le \epsilon$, then $|f'(x) - g'(x)| \le
\epsilon$ with the construction~\eqref{eqn:subgradient-oracle} of
subgradient oracle.
\begin{lemma}
  \label{lemma:chi-square-gradients}
  Let the subgradient oracle be given by~\eqref{eqn:subgradient-oracle},
  and let $P_f^T$ and $P_g^T$ be the distributions (respectively)
  of the observed stochastic sub-gradients
  \begin{equation}
    s_i = f'(x_i) + \noise_i
    ~~ \mbox{or} ~~
    s_i = g'(x_i) + \noise_i,
  \end{equation}
  where $x_i$ is a measurable function of an independent noise variable
  $\statrv_0$ and the preceding sequence of stochastic gradients $\{s_1,
  \ldots, s_{i-1}\}$.  Let $\fnmetric(f, g) \le \epsilon$. Then
  \begin{align}
    \dchii{P_f^T}{P_g^T}
    &\le \exp\left(\frac{T \epsilon^2}{\stddev^2}\right) - 1.
  \end{align}
\end{lemma}
\begin{proof}
  Let $s_i$ be the $i$th observed stochastic subgradient in the sequence,
  and let the $\sigma$-field
  of the observed sequence through time $i$ be 
  $\mc{F}_i = \sigma(\statrv_0, s_1, \ldots, s_i)$. Then we have
  \begin{align}
    \dchii{P_f^T}{P_g^T} + 1
    & = \int \frac{dP_f^T(s_{1:n})}{dP_g^T(s_{1:n})}
    dP_f^T(s_{1:n}) \\
    & = \int \prod_{i = 1}^T \left[
      \frac{dP_f(s_i \mid s_{1:i-1})}{
      dP_g(s_i \mid s_{1:i-1})}
      dP_f(s_i \mid s_{1:i-1})\right] \\
    & = \E\left[
      \prod_{i = 1}^T \E_{P_f} \left[
        \frac{dP_f(S_i \mid \mc{F}_{i-1})}{
          dP_g(S_i \mid \mc{F}_{i-1})} \mid \mc{F}_{i-1}
        \right]\right].
  \end{align}
  By the measurability assumption on $x_i$, that is,
  $x_i \in \mc{F}_{i-1}$, the inner expectation is simply one plus the
  $\chi^2$ distance between two distributions
  $\normal(f'(x_i), \stddev^2)$ and $\normal(g'(x_i), \stddev^2)$,
  which we know satisfies
  \begin{equation}
    \E_{P_f} \left[
      \frac{dP_f(S_i \mid \mc{F}_{i-1})}{
        dP_g(S_i \mid \mc{F}_{i-1})} \mid \mc{F}_{i-1}
      \right]
    = \exp\left(\frac{(f'(x_i) - g'(x_i))^2}{\stddev^2}\right)
    \le \exp\left(\frac{\epsilon^2}{\stddev^2}\right).
  \end{equation}
  Taking the product over all $T$ terms yields the lemma.
\end{proof}

\begin{lemma}
  Let $f$ be a closed convex function.
  Define the function
  \begin{equation}
    \begin{split}
      H(\epsilon) & \defeq
      \inf \left\{|x - x_0| : x \in \partial f^*(\epsilon),
      x_0 \in \partial f^*(0) \right\}
      \vee \inf \left\{|x - x_0| : x \in \partial f^*(-\epsilon),
      x_0 \in \partial f^*(0) \right\} \\
      & \; = \dist(\partial f^*(\epsilon), \partial f^*(0))
      \vee \dist(\partial f^*(-\epsilon), \partial f^*(0)).
    \end{split}
    \label{eqn:def-big-H}
  \end{equation}
  For any $0 \le c_l < 1$ and $1 < c_u < \infty$,
  \begin{equation}
    \modcont_f(c_u \epsilon) \ge
    H(\epsilon) \ge \modcont_f(c_l \epsilon).
  \end{equation}
\end{lemma}

\begin{proposition}
  \label{proposition:basic-super-efficiency}
  Define $H$ to be the function~\eqref{eqn:def-big-H} and
  assume additionally that
  $\delta < \sqrt{\frac{1}{8 e}}$.
  If $\what{x}$ is any estimator such that
  \begin{equation}
    \sqrt{\E_{P_f^T}\left[\dist(\hat x,\mc{X}_f^*)^2\right]}
    \le \delta \modcont_f(\stddev / \sqrt{T}),
  \end{equation}
  then taking $f_1(x) = f(x) + \sqrt{\frac{\stddev^2 \log
      \frac{1}{8\delta^2}}{T}} x$ and $f_{-1}(x) = f(x) -
  \sqrt{\frac{\stddev^2 \log \frac{1}{8\delta^2}}{T}} x$,
  we have
  \begin{align}
    \max_{g \in \{f_1, f_{-1}\}} 
    \E_{P_g^T}\left[\dist(\hat x,\mc{X}_g^*)\right]
     \ge
 & \sup_{0 < c < \log \frac{1}{8 \delta^2}} \modcont_f\left(
    \sqrt{\frac{c \stddev^2}{T}}\right)
    \left(1 - \frac{\modcont_f(\stddev / \sqrt{T})}{
      2\sqrt{2} \modcont_f(\sqrt{c \stddev^2 / T})}\right) \\
    & \ge \frac{4 - \sqrt{2}}{4}
    H\left(\sqrt{\frac{\stddev^2 \log \frac{1}{8 \delta^2}}{T}}\right).
  \end{align}
\end{proposition}
\begin{proof}
  Without loss of generality, we assume that $0 \in \argmin_x f(x) =
  \partial f^*(0)$, and set $x_0 = 0$ for simplicity in the derivation. For
  any $\epsilon \in \R$, we may construct the function $f_\epsilon(x) = f(x)
  + \epsilon x$.  Lemma~\ref{lemma:brown-low-stronger} and
  Lemma~\ref{lemma:chi-square-gradients} thus yield that that for
  $\mc{X}_\epsilon = \argmin_x f_\epsilon(x)$, we have
  \begin{align}
    \E_{P_{f_\epsilon}^T}[\dist(\what{x}, \mc{X}_\epsilon)]
    & \ge \dist(\partial f^*(-\epsilon), \partial f^*(0))
    \hinge{1 - \frac{ \modcont_f(\stddev / \sqrt{T})
        \sqrt{\delta \exp(\frac{T \epsilon^2}{\stddev^2})}}{
        \dist(\partial f^*(-\epsilon), \partial f^*(0))}}
  \end{align}
  and
  \begin{align}
    \E_{P_{f_{-\epsilon}}^T}[
      \dist(\what{x}, \mc{X}_{-\epsilon})]
    & \ge \dist(\partial f^*(\epsilon), \partial f^*(0))
    \hinge{1 - \frac{\modcont_f(\stddev / \sqrt{T})
        \sqrt{\delta \exp(\frac{T \epsilon^2}{\stddev^2})}}{
        \dist(\partial f^*(\epsilon), \partial f^*(0))}}.
  \end{align}
  In particular, with $H(\epsilon) = \dist(\partial f^*(\epsilon), \partial f^*(0))
  \vee \dist(\partial f^*(\epsilon), \partial f^*(0))$, we have
  \begin{equation}
    \max_{g \in f_\epsilon, f_{-\epsilon}}
    \E_{P_g^T}\left[\dist(\what{x}, \mc{X}_g^*)\right]
    \ge H(\epsilon) \hinge{1 - \frac{ \modcont_f(\stddev / \sqrt{T})
        \sqrt{\delta \exp(\frac{n \epsilon^2}{\stddev^2})}}{
        H(\epsilon)}}.
    \label{eqn:g-lower-bound}
  \end{equation}
  Take $\epsilon^2 = \frac{\stddev^2}{T} \log \frac{1}{8 \delta^2}$
  to obtain
  \begin{equation}
    \max_{g \in f_\epsilon, f_{-\epsilon}}
    \E_{P_g^T}\left[\dist(\what{x}, \mc{X}_g^*)\right]
    \ge H\left(\sqrt{\frac{\stddev^2 \log \frac{1}{8 \delta^2}}{T}}\right)
    \hinge{1 - \frac{\modcont_f(\stddev / \sqrt{T})}{
        2\sqrt{2} H(\stddev \log^\half \frac{1}{8 \delta^2} / \sqrt{T})}}.
  \end{equation}
  Notably, by Lemma~\ref{lemma:modcont}, our w.l.o.g.\ assumption
  and the fact that
  subgradients are increasing, we have that for any constant
  $(\log \frac{1}{8\delta^2})^{-\half} \le c < 1$ that
  \begin{align}
    \hskip-5pt
    & \modcont_f\left(\frac{\stddev}{\sqrt{T}}\right)
     \le \modcont_f\left(c \sqrt{\frac{\stddev^2 \log \frac{1}{8\delta^2}}{T}}
    \right) \\
    & \hskip-5pt \le \sup\left\{\dist(x, X_0) : x \in \partial f^*\left(
    c \frac{\stddev \log^\half \frac{1}{8\delta^2}}{\sqrt{T}}\right)\right\}
    \vee
    \sup\left\{\dist(x, X_0) : x \in \partial f^*\left(
    -c \frac{\stddev \log^\half \frac{1}{8\delta^2}}{\sqrt{T}}\right)\right\} \\
    & \le \sup\left\{\dist(x, X_0) : x \in \partial f^*\left(
    \frac{\stddev \log^\half \frac{1}{8\delta^2}}{\sqrt{T}}\right)\right\}
    \vee
    \sup\left\{\dist(x, X_0) : x \in \partial f^*\left(
    - \frac{\stddev \log^\half \frac{1}{8\delta^2}}{\sqrt{T}}\right)\right\} \\
    & = H\left(\frac{\stddev \log^\half \frac{1}{8\delta^2}}{\sqrt{T}}\right).
  \end{align}
  In particular,  we have the lower bound
  \begin{equation}
    \max_{g \in f_\epsilon, f_{-\epsilon}}
    \E_{P_g^T}\left[\dist(\what{x}, \mc{X}_g^*)\right]
    \ge H\left(\sqrt{\frac{\stddev^2 \log \frac{1}{8 \delta^2}}{T}}\right)
    \frac{4 - \sqrt{2}}{4}.
  \end{equation}
  This is the desired result.
\end{proof}

Proposition~\ref{proposition:basic-super-efficiency} is a basic result on
superefficiency that we may specialize to obtain more concrete results. We
would like give a result that holds when $f^*$ is differentiable in a
neighborhood of $0$, which is equivalent to $f$ being strictly convex in a
neighborhood of $x_0 = \argmin_x f(x)$, by
Lemma~\ref{lemma:subgradients}. This would mean that the function $H$
defined in Proposition~\ref{proposition:basic-super-efficiency} satisfies
\begin{equation}
  H(\epsilon) = \max\{|{f^*}'(\epsilon) - x_0|, |{f^*}'(-\epsilon) - x_0|\}
  = \modcont_f(\epsilon)
\end{equation}
for all small enough $\epsilon > 0$.
In this setting, we obtain
\begin{corollary}
  Let the conditions of Proposition~\ref{proposition:basic-super-efficiency}
  hold, and let $f$ be strictly convex in a neighborhood of
  $x_0 = \argmin_x f(x)$. Assume that $\what{x}$ is any estimator
  satisfying
  \begin{equation}
    \sqrt{\E_{P_f^T}\left[(\what{x} - x_0)^2\right]}
    \le \delta \modcont_f(\stddev / \sqrt{T}),
  \end{equation}
  where $\delta < \sqrt{\frac{1}{8e}}$. Define $f_{\pm 1}(x) = f(x) \pm
  \sqrt{\frac{\stddev^2 \log \frac{1}{8\delta^2}}{T}} x$.
  Then for large enough $T$,
  \begin{equation}
    \max_{g \in \{f_1, f_{-1}\}}
    \E_{P_g^T}|\what{x} - x_g\opt|
    \ge \frac{4 - \sqrt{2}}{4}
    \modcont_f\left(\sqrt{\frac{\stddev^2 \log \frac{1}{8\delta^2}}{T}}\right).
  \end{equation}
\end{corollary}
This corollary has a striking weakness, however---the right hand side
depends on $\modcont_f$, rather than $\modcont_g$, which is what we would
prefer.  We can, however, state a simpler result that is achievable.
\begin{corollary}
  \label{corollary:good-asymptotic-super-efficiency}
  Let $f$ be any convex function satisfying the asymptotic
  expansion~\eqref{eqn:local-growth} around its optimum.
  Suppose that $\what{x}$ is any estimator such that
  \begin{equation}
    \sqrt{\E_{P_f^T}[\dist(\what{x}, \mc{X}_f^*)^2]}
    \le \delta \modcont_f\left(\frac{\stddev}{\sqrt{T}}\right),
  \end{equation}
  where $\delta < \sqrt{\frac{1}{8 e}}$.
  Define $g_{-1}(x) = f(x) - \epsilon_T x$ and $g_1(x) = f(x) + \epsilon_T x$,
  where $\epsilon_T = \sqrt{\frac{\stddev^2 \log \frac{1}{8\delta^2}}{T}}$,
  and let $\pow = \pow_r \vee \pow_l$.
  Let $C > 1$ and $0 < c < 1$ be otherwise arbitrary numerical constants.
  Then for one of $g \in \{g_{-1}, g_1\}$, there exists
  $T_0$ such that $T \ge T_0$ implies
  \begin{equation}
    \E_{P_g}\left[\dist(\what{x}, \mc{X}_g^*)\right]
    \ge \frac{4 - \sqrt{2}}{4 (2C^2)^\frac{1}{\pow - 1} e}
    \, \modcont_g\left(
    c \sqrt{\frac{\stddev^2 \log \frac{1}{8 \delta^2}}{T}}
    \right).
  \end{equation}
\end{corollary}
\begin{proof}
  Without loss of generality, we assume that
  $\pow_r \ge \pow_l$, and if $\pow_l = \pow_r$ then $\lambda_r \ge
  \lambda_l$.
  By inspection of the proof of
  Proposition~\ref{proposition:basic-super-efficiency},
  we have that
  \begin{equation}
    \E_{P_{g_{-1}}^T}[\dist(\what{x}, \partial g_{-1}^*(0))]
    \ge \half \dist(\partial f^*(\epsilon_T), \partial f^*(0)).
  \end{equation}
  Moreover, we know that for suitably large $n$,
  we have by Lemma~\ref{lemma:growth-f*}
  \begin{align}
    \dist(\partial f^*(\epsilon_T), \partial f^*(0))
    & = \dist(\partial f^*(\epsilon_T), \partial f^*(0))
    \vee \dist(\partial f^*(-\epsilon_T), \partial f^*(0)) \\
    & \ge \modcont_f(c \epsilon_T)
  \end{align}
  for any $c < 1$. Then Lemma~\ref{lemma:modconts-growth}
  implies that for any $C > 1$, there exists $T_0$ such that $T \ge T_0$ implies
  \begin{equation}
    \modcont_f(c \epsilon_T)
    \ge
    \frac{1}{(2C^2)^\frac{1}{\pow - 1} e}
    \modcont_{g_{-1}}(c^2 \epsilon_T).
  \end{equation}
  This gives the desired result.
\end{proof}

As an immediate consequence of
Corollary~\ref{corollary:good-asymptotic-super-efficiency}, we see that if
there exists any sequence $\delta_T \to 0$ with $\liminf_T e^T \delta_T =
\infty$ such that
\begin{equation}
  \sqrt{\E_{P_f}\left[\dist(\what{x}, \mc{X}_f^*)^2\right]}
  \le \delta_T \modcont_f\left(\frac{\stddev}{\sqrt{T}}\right),
\end{equation}
then there exists a sequence of convex functions $g_T$,
with $\fnmetric(f, g_T) \to 0$, such that
\begin{equation}
  \liminf_T \frac{\E_{P_{g_T}}
    \left[\dist(\what{x}, \mc{X}_{g_T})\right]}{
    \modcont_{g_T}\left(\sqrt{\frac{\stddev^2
        \log {\delta_T}^{-1}}{T}}\right)}
  > 0.
\end{equation}

\section{Algorithm }

\subsection{Proof of Proposition \ref{prop:algorithm}}
First, by the monotonicity of the derivative $f'$, note that the interval $\mc{I}_\delta$ is such that $x\in \mc{I}_\delta$ holds if and only if $|f'(x)|< C_\delta/\sqrt{T_0}$.
Now suppose that at round $e$, $(a_e,b_e)\cap \mc{I}_\delta\neq \emptyset$.
For the next round, if $x_e=(a_e+b_e)/2\in\mc{I}_\delta$, then $(a_{e+1},b_{e+1})\cap\mc{I}_\delta\neq\emptyset$.
Otherwise, if $x_e\notin \mc{I}_\delta$, we know that $|f'(x_e)|\geq C_\delta/\sqrt{T_0}$, and without loss of generality, we assume that it is positive.
Then, we have 
\begin{align}
\mathbb P\left((a_{e+1},b_{e+1})\cap\mc{I}_\delta\neq\emptyset\right)&=
\mathbb P\left(\normal\left(f'(x_e),\frac{\sigma^2}{T_0}\right)<0\right)
=\mathbb P\left(\normal(0,1)>\frac{\sqrt{T_0}f'(x_e)}{\sigma}\right)\\
&\leq \mathbb P\left(\normal(0,1)>\frac{C_\delta}{\sigma}\right)
\leq\frac{\sigma}{C_\delta\sqrt{2\pi}}\exp\left(-\frac{C_\delta^2}{2\sigma^2}\right)
\end{align}
Therefore,
\begin{equation}
\mathbb P\left((a_{e+1},b_{e+1})\cap \mc{I}_\delta\neq \emptyset\big|(a_{e},b_{e})\cap \mc{I}_\delta\neq \emptyset\right)\geq 1-\frac{\sigma}{C_\delta\sqrt{2\pi}}\exp\left(-\frac{C_\delta^2}{2\sigma^2}\right)
\end{equation}
It then follows that 
\begin{align}
\mathbb P\left((a_E,b_E)\cap \mc{I}_\delta\neq \emptyset\right)&=\mathbb P\left((a_e,b_e)\cap \mc{I}_\delta\neq \emptyset\text{ for }e=1,\dots, E\right)\\
&=\prod_{e=0}^{E-1}\mathbb P\left((a_{e+1},b_{e+1})\cap \mc{I}_\delta\neq \emptyset\big|(a_e,b_e)\cap \mc{I}_\delta\neq \emptyset\right)\\
&\geq \left(1-\frac{\sigma}{C_\delta\sqrt{2\pi}}\exp\left(-\frac{C_\delta^2}{2\sigma^2}\right)\right)^E\\
&\geq 1-\frac{E\sigma}{C_\delta\sqrt{2\pi}}\exp\left(-\frac{C_\delta^2}{2\sigma^2}\right)\\
&\geq 1-\delta
\end{align}
by the choice of $C_\delta$.

\subsection{Proof of Corollary \ref{cor:algorithm}}
By the polynomial growth condition, we have for $T>\sigma^2/\epsilon_0$,
\begin{align}
\omega_f(\epsilon_0)\leq \left(\frac{\epsilon_0\sqrt{T}}{\sigma}\right)^\alpha\omega_f\left(\frac{\sigma}{\sqrt{T}}\right).
\end{align}
Since $r=\half\alpha_0\geq\half\alpha$ and $E=\lfloor r\log T\rfloor$,
\begin{align}
2^{-E}(b_0-a_0)\leq 2(b_0-a_0)T^{-r}\leq 2(b_0-a_0)T^{-\half\alpha}
\leq \frac{2(b_0-a_0)\epsilon_0^\alpha}{\omega_f(\epsilon_0)\sigma^\alpha}\omega_f\left(\frac{\sigma}{\sqrt{T}}\right)
\end{align}
By the expression we obtained in Example~\ref{exmp:one-dim-modulus},
\begin{align}
&\sup\{\inf_{x\in\mc{X}_f^*}|x-y|:y\in\mc{I}_\delta\}\\
&=\omega_f\left(\frac{C_\delta}{\sqrt{T_0}}\right)
\leq \left(\sqrt{2r\left(\log(r\log T)+\log\frac{1}{\delta}\right)\log T}\right)^{\alpha}\omega_f\left(\frac{\sigma}{\sqrt{T}}\right)
\end{align}
for $T$ large enough. 
Therefore, we obtain that there exist $T'>0$ such that for $T>T'$,
\begin{equation}
\inf_{x\in\mc{X}_f^*} |x_E-x|\leq \tilde C\omega_f\left(\frac{1}{\sqrt{T}}\right)
\end{equation}
where
\begin{equation}
\tilde C = \frac{2(b_0-a_0)\epsilon_0^\alpha}{\omega_f(\epsilon_0)\sigma^\alpha} + \left(\sqrt{2r\left(\log(r\log T)+\log\frac{1}{\delta}\right)\log T}\right)^{\alpha}.
\end{equation}

\end{appendix}

\section*{Acknowledgments}


Research supported in part by ONR grant 11896509 and NSF grant
DMS-1513594.

\setlength{\bibsep}{8pt}
\bibliographystyle{apalike}
\bibliography{local_minimax}

\end{document}